\documentclass{article} 
\usepackage{nips11submit_e,times}

\usepackage[english]{babel}
\usepackage[normalem]{ulem}	
\usepackage{subfigure}
\usepackage{amsmath, amssymb,amsthm}

\newtheorem{theorem}{Theorem}
\newtheorem{lemma}{Lemma}
\newtheorem{definition}{Definition}
\newtheorem{corollary}{Corollary}

\def\U{\mathcal{U}}

\def\R{\mathbb{R}}

\def\E{\mathbb{E}}
\def\P{\mathbb{P}}
\def\1{\mathbf{1}}

\usepackage{array}
\usepackage[belowskip=-10pt,aboveskip=2pt]{caption}

\usepackage{multirow}
\usepackage{graphicx}
\usepackage{color}
\title{Active Ranking using Pairwise Comparisons}

\author{Kevin G. Jamieson\\
University of Wisconsin\\
Madison, WI 53706, USA \\
\texttt{\small kgjamieson@wisc.edu}
\And 
Robert D. Nowak\\
University of Wisconsin\\
Madison, WI 53706, USA \\
\texttt{\small nowak@engr.wisc.edu}
}

\nipsfinalcopy 

\begin{document}

\maketitle

\begin{abstract}
  This paper examines the problem of ranking a collection of objects using pairwise comparisons (rankings of two objects).  In general, the ranking of $n$ objects can be identified by standard sorting methods using $n\log_2 n$ pairwise comparisons. We are interested in natural situations in which relationships among the objects may allow for ranking using far fewer pairwise comparisons. {Specifically, we assume that the objects can be embedded into a $d$-dimensional Euclidean space and that the rankings reflect their relative distances from a common reference point in $\R^d$. We show that under this assumption the number of possible rankings grows like $n^{2d}$ and demonstrate an algorithm that can identify a randomly selected ranking using just slightly more than $d\log n$ adaptively selected pairwise comparisons, on average.}  If instead the comparisons are chosen at random, then almost all pairwise comparisons must be made in order to identify any ranking. In addition, we propose a robust, error-tolerant algorithm that only requires that the pairwise comparisons are probably correct. Experimental studies with synthetic and real datasets support the conclusions of our theoretical analysis. \end{abstract}

\section{Introduction}

This paper addresses the problem of ranking a set of objects based on a limited number of pairwise comparisons (rankings between pairs of the objects).  A {\em ranking} over a set of $n$ objects $\Theta= (\theta_1,\theta_2,\dots,\theta_n)$ is a mapping $\sigma: \{1,\dots,n\} \rightarrow \{1,\dots,n\} $ that prescribes an order \begin{align} \sigma(\Theta) & :=  \theta_{\sigma(1)} \prec \theta_{\sigma(2)} \prec \dots \prec \theta_{\sigma({n-1})}\prec \theta_{\sigma({n})} \label{sigma} \end{align} where $\theta_i \prec \theta_j$ means $\theta_i$ precedes $\theta_j$ in the ranking. {A ranking uniquely determines the collection of pairwise comparisons between all pairs of objects}. The primary objective here is to bound the number of pairwise comparisons needed to correctly determine the ranking when the objects (and hence rankings) satisfy certain known structural constraints.  Specifically, we suppose that the objects may be embedded into a low-dimensional Euclidean space such that the ranking is consistent with distances in the space.  We wish to exploit such structure in order to discover the ranking using a very small number of pairwise comparisons. To the best of our knowledge, this is a {previously} open and unsolved problem.

There are practical and theoretical motivations for restricting our attention to pairwise rankings that are discussed in Section~\ref{motivation}.  We begin by assuming that every pairwise comparison is consistent with an unknown ranking.  Each pairwise comparison can be viewed as a query: is $\theta_i$ before $\theta_j$?  Each query provides $1$ bit of information about the underlying ranking.  Since the number of rankings is $n!$, in general, specifying a ranking requires $\Theta(n\log n)$ bits of information.  This implies that at least this many pairwise comparisons are required without additional assumptions about the  ranking.  In fact, this lower bound can be achieved with a standard adaptive sorting algorithm like binary sort \cite{knuth}. In large-scale problems or when humans are queried for pairwise comparisons, obtaining this many pairwise comparisons may be impractical and therefore we consider situations in which the space of rankings is structured and thereby less complex. 

A natural way to induce a structure on the space of rankings is to suppose that the objects can be embedded into a $d$-dimensional Euclidean space so that the distances between objects are consistent with the ranking.  This may be a reasonable assumption in many applications, and for instance the audio dataset used in our experiments is believed to have a 2 or 3 dimensional embedding \cite{philips}.  We further discuss motivations for this assumption in Section~\ref{motivation}. It is not difficult to show (see Section~\ref{geometry}) that the number of full rankings that could arise from $n$ objects embedded in $\R^d$ grows like $n^{2d}$, and so specifying a ranking from this class requires only $O(d\log n)$ bits.  The main results of the paper show that under this assumption a randomly selected ranking can be determined using $O(d\log n$) pairwise comparisons selected in an adaptive and sequential fashion, but almost all $\binom{n}{2}$ pairwise rankings are needed if they are picked randomly rather than selectively.  In other words, {\em actively selecting the most informative queries has a tremendous impact on the complexity of learning the correct ranking.}

\subsection{Problem statement}
\label{ps}

Let $\sigma$ denote the ranking to be learned. The objective is to learn the ranking by querying the reference for pairwise comparisons of the form
\begin{align}
q_{i,j} & :=  \{ \theta_i \prec \theta_j \}.
\label{query}
\end{align}
The response or label of $q_{i,j}$ is binary and denoted as $y_{i,j} := \1\{q_{i,j}\}$ where $\1$ is the indicator function; ties are not allowed.  The main results quantify the minimum number of queries or labels required to determine the reference's ranking, and they are based on two key assumptions. 

\textbf{A1 Embedding:} {The set of $n$ objects are embedded in $\R^d$ (in general position) and we will also use $\theta_1,\dots,\theta_n $ to refer to their (known) locations in $\R^d$. Every ranking $\sigma$ can be specified by a {\em reference} point $r_{\sigma} \in \R^d$, as follows. The Euclidean distances between the reference and objects are consistent with the ranking in the following sense: if the $\sigma$ ranks $\theta_i \prec \theta_j$, then $\|\theta_i-r_{\sigma}\| < \|\theta_j - r_{\sigma}\|$. Let $\Sigma_{n,d}$ denote the set of all possible rankings of the $n$ objects that satisfy this embedding condition. }

{The interpretation of this assumption is that we know how the objects are related (in the embedding), which limits the space of possible rankings.  The ranking to be learned, specified by the reference (e.g., preferences of a human subject), is unknown. Many have studied the problem of finding an embedding of objects from data \cite{mcfee,gormley,cox}. This is not the focus here, but it could certainly play a supporting role in our methodology (e.g., the embedding could be determined from known similarities between the $n$ objects, as is done in our experiments with the audio dataset).  We assume the embedding is given and our interest is minimizing the number of queries needed to learn the ranking, and for this we require a second assumption.}

\textbf{A2 Consistency:} Every pairwise comparison is consistent with the  ranking to be learned.  That is, if the reference ranks $\theta_i \prec \theta_j$, then $\theta_i$ must precede $\theta_j$ in the (full) ranking. 

As we will discuss later in Section~\ref{LBQC}, these two assumptions alone are not enough to rule out pathological arrangements of objects in the embedding for which at least $\Omega(n)$ queries must be made to recover the ranking. However, because such situations are not representative of what is typically encountered, we analyze the problem in the framework of the average-case analysis \cite{traub}.  
{\begin{definition}  \label{distributionDef}
With each ranking $\sigma \in \Sigma_{n,d}$ we associate a probability $\pi_{\sigma}$ such
that $\sum_{\sigma \in \Sigma_{n,d}} \pi_{\sigma} = 1$. Let $\pi$ denote these probabilities
and write $\sigma \sim \pi$ for shorthand.  The uniform distribution corresponds to $\pi_\sigma = |\Sigma_{n,d}|^{-1}$ for all $\sigma \in \Sigma_{n,d}$, and we write $\sigma \sim \U$ for this special case.
\end{definition}
\begin{definition}
If $M_n(\sigma)$ denotes the number of pairwise comparisons requested by an algorithm to identify the ranking $\sigma$, then the {\em average query complexity} with respect to $\pi$ is denoted by $\E_\pi [ M_n ]$.  \vspace{-5pt}
\end{definition}}
{The main results are proven for the special case of $\pi=\U$, the uniform distribution, to make the analysis more transparent and intuitive. However the results can easily be extended to general distributions $\pi$ that satisfy certain mild conditions (see Appendix~\ref{simpleTheoremProof}).} All results henceforth, unless otherwise noted, will be given in terms of (uniform) average query complexity and we will say such results hold ``on average.'' 

Our main results can be summarized as follows. If the queries are chosen deterministically or randomly in advance of collecting the corresponding pairwise comparisons, then we show that almost all $\binom{n}{2}$ pairwise comparisons queries are needed to identify a ranking under the assumptions above. {However, if the queries are selected in an adaptive and sequential fashion according to the algorithm in Figure~\ref{fig:alg}, then we show that the number of pairwise rankings required to identify a ranking is no more than a constant multiple of $d \log n$, on average.} The algorithm requests a query if and only if the corresponding pairwise ranking is ambiguous (see Section~\ref{ambigDef}), meaning that it cannot be determined from previously collected pairwise comparisons and the locations of the objects in $\R^d$.  The efficiency of the algorithm is due to the fact that {\em most} of the queries are unambiguous when considered in a sequential fashion.  For this very same reason, picking queries in a non-adaptive or random fashion is very inefficient. It is also noteworthy that the algorithm is also computationally efficient with an overall complexity no greater than $O( n \ \text{poly}(d) \ \text{poly}(\log n) )$ (see Appendix~\ref{complexity}). In Section~\ref{robust} we present a robust version of the algorithm of Figure~\ref{fig:alg} that is tolerant to a fraction of errors in the pairwise comparison queries. In the case of {\em persistent errors} (see Section~\ref{robust}) we show that we can find a probably approximately correct ranking by requesting just $O(d \log^2 n)$ pairwise comparisons. This allows us to handle situations in which either or both of the assumptions, \textbf{A1} and \textbf{A2}, are reasonable approximations to the situation at hand, but do not hold strictly (which is the case in our experiments with the audio dataset). 

Proving the main results involves an uncommon marriage of ideas from the ranking and statistical learning literatures.  Geometrical interpretations of our problem derive from the seminal works of \cite{coombs} in ranking and \cite{cover} in learning.  From this perspective our problem bears a strong resemblance to the halfspace learning problem, with two crucial distinctions.  In the ranking problem, the underlying halfspaces are not in general position and have strong dependencies with each other. These dependencies invalidate many of the typical analyses of such problems \cite{dasgupta, hanneke}. One popular method of analysis in exact learning involves the use of something called the {\em extended teaching dimension} \cite{hegedus}. However, because of the possible pathological situations alluded to earlier, it is easy to show that the extended teaching dimension must be at least $\Omega(n)$ making that {sort of worst-case}  analysis uninteresting. These differences present unique challenges to learning.

\begin{figure}%
\centering
\parbox{2.5in}{%
\fbox{\parbox[b]{2.4in}{{\underline{\bf Query Selection Algorithm}}  \\ \vspace{-.1in} \\
input: $n$ objects in $\R^d$\\
initialize: objects $\Theta = \{\theta_1,\dots,\theta_n\}$ in uniformly random order \\ \vspace{-.1in} \\
for j=2,\dots,n \\
\mbox{ \   } for i=1,\dots,j-1 \\
\mbox{ \ \ \ \ \bf if} $q_{i,j}$ is {\em ambiguous}, \\
\mbox{ \ \ \ \ \ } request $q_{i,j}$'s label from reference; \\
\mbox{ \ \ \ \ \bf else} \\
\mbox{ \ \ \ \ \ } impute $q_{i,j}$'s label from previously\\
\mbox{ \ \ \ \ \ }  labeled queries. \\ \vspace{-.1in} \\
output: ranking of $n$ objects
}}
\label{fig:alg}
\caption{Sequential algorithm for selecting queries. See Figure~\ref{ambiguous} and Section~\ref{ambigDef} for the definition of an ambiguous query. }\label{fig:alg}
\vspace{.06in}
}%
\qquad
\begin{minipage}{2.6in}%
\begin{center}
\includegraphics[scale=.325]{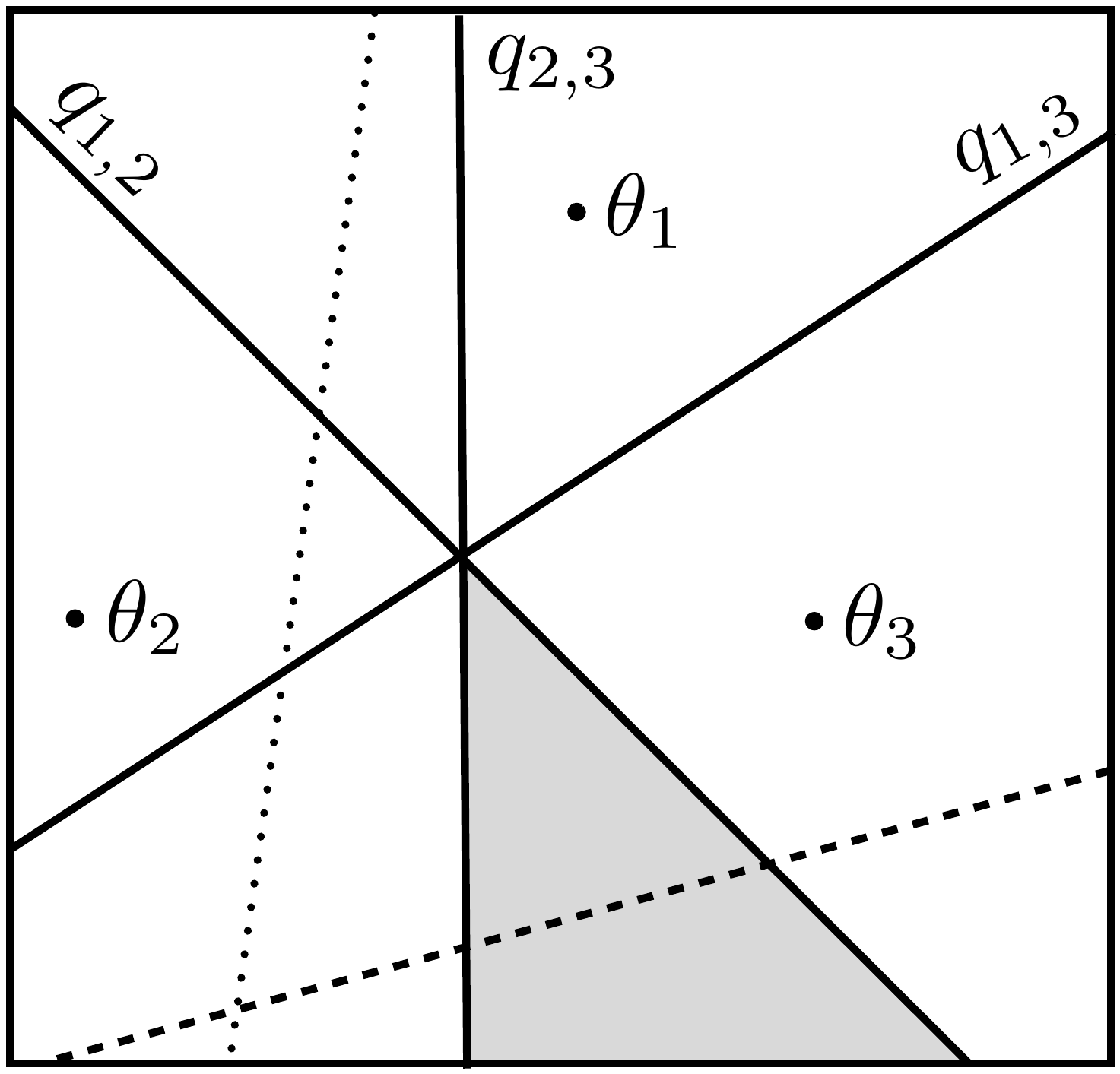} 
\end{center}
\caption{Objects $\theta_1,\theta_2,\theta_3$ and queries. The $r_\sigma$ lies in the shaded region (consistent with the labels of $q_{1,2},q_{1,3},q_{2,3}$).  The dotted (dashed) lines represent new queries whose labels are (are not) ambiguous given those labels.}
\label{ambiguous}
\end{minipage}
\vspace{-3pt}
\end{figure}%

\section{Motivation and related work} \label{motivation} 
The problem of learning a ranking from few pairwise comparisons is motivated by what we perceive as a significant gap in the theory of ranking and permutation learning. {Most work in ranking with structural constraints assumes a passive approach to learning; pairwise comparisons or partial rankings are collected in a random or non-adaptive fashion and then  aggregated to obtain a full ranking (cf. \cite{freund, burges, zheng, herbrich}).  However, this may be quite inefficient in terms of the number of pairwise comparisons or partial rankings needed to learn the (full) ranking.  This inefficiency was recently noted in the related area of social choice theory \cite{lu}.   Furthermore, empirical evidence suggests that, even under complex ranking models,  adaptively selecting pairwise comparisons can reduce the number needed to learn the ranking \cite{chu}.} It is cause for concern since in many applications it is expensive and time-consuming to obtain pairwise comparisons. For example, psychologists and market researchers collect pairwise comparisons to gauge human preferences over a set of objects, for scientific understanding or product placement. The scope of these experiments is often very limited simply due to the time and expense required to collect the data.  This suggests the consideration of more selective and judicious approaches to gathering inputs for ranking.  We are interested in taking advantage of underlying structure in the set of objects in order to choose more informative pairwise comparison queries. From a learning perspective, our work adds an active learning component to a problem domain that has primarily been treated from a passive learning mindset.  

We focus on pairwise comparison queries for two reasons.  First, pairwise comparisons admit a halfspace representation in embedding spaces which allows for a geometrical approach to learning in such structured ranking spaces.  Second, pairwise comparisons are the most common form of queries in many applications, especially those involving human subjects. For example, consider the problem of finding the most highly ranked object, as illustrated by the following familiar task. Suppose a patient needs a new pair of prescription eye lenses. Faced with literally millions of possible prescriptions, the doctor will present candidate prescriptions in a sequential fashion followed by the query: better or worse? Even if certain queries are repeated to account for possible inaccurate answers, the doctor can locate an accurate prescription with just a handful of queries. This is possible presumably  because the doctor understands (at least intuitively) the intrinsic space of prescriptions and can efficiently search through it using only binary responses from the patient.

We assume that the objects can be embedded in $\R^d$ and that the distances between objects and the reference are consistent with the ranking (Assumption {\bf A1}). The problem of learning a general function $f:  \R^d \rightarrow \R$ using just pairwise comparisons that correctly ranks the objects embedded in $\R^d$ has previously been studied in the passive setting \cite{freund, burges, zheng,herbrich}. The main contributions of this paper are theoretical bounds for the specific case when $f(x) = ||x - r_{\sigma}||$ where $r_{\sigma} \in \R^d$ is the reference point. This is a standard model used in multidimensional unfolding and psychometrics \cite{coombs,bennett} and one can show that this model also contains the familiar functions $f(x) = r_{\sigma}^T x$ for all $r_{\sigma} \in \R^d$. We are unaware of any existing query-complexity bounds for this problem.  {We do not assume a generative model is responsible for the relationship between rankings to embeddings, but one could.  For example, the objects might have an embedding (in a feature space) and the ranking is generated by distances in this space.  Or alternatively, structural constraints on the space of rankings could be used to generate a consistent embedding.} Assumption {\bf A1}, while arguably quite natural/reasonable in many situations, significantly constrains the set of possible rankings.

\section{Geometry of rankings from pairwise comparisons} \label{geometry}

The embedding assumption \textbf{A1} gives rise to geometrical interpretations of the ranking problem, which are developed in this section.  {The pairwise comparison $q_{i,j}$ can be viewed as the membership query: is $\theta_i$ ranked before $\theta_j$ in the (full) ranking $\sigma$? The geometrical interpretation is that $q_{i,j}$ requests whether the reference $r_{\sigma}$ is closer to object $\theta_i$ or object $\theta_j$ in $\R^d$.} Consider the line connecting $\theta_i$ and $\theta_j$ in $\R^d$.  The hyperplane that bisects this line and is orthogonal to it defines two halfspaces: one containing points closer to $\theta_i$ and the other the points closer to $\theta_j$.  Thus, $q_{i,j}$ is a membership query about which halfspace $r_{\sigma}$ is in, and there is an equivalence between each query, each pair of objects, and the corresponding bisecting hyperplane. The set of all possible pairwise comparison queries can be represented as $\binom{n}{2}$ distinct halfspaces in $\R^d$.  The intersections of these halfspaces partition $\R^d$ into a number of cells, and each one corresponds to a unique ranking of $\Theta$.  Arbitrary rankings are not possible due to the embedding assumption \textbf{A1}, and recall that the set of rankings possible under \textbf{A1} is denoted by $\Sigma_{n,d}$.  The cardinality of $\Sigma_{n,d}$ is equal to the number of cells in the partition. We will refer to these cells as $d$-cells (to indicate they are subsets in $d$-dimensional space) since at times we will also refer to lower dimensional cells; e.g., $(d-1)$-cells.

\subsection{Counting the number of possible rankings}

The following lemma determines the cardinality of the set of rankings, $\Sigma_{n,d}$, under assumption \textbf{A1}.

\begin{lemma} \cite{coombs} \label{counting}
Assume \textbf{A1-2}. Let $Q(n,d)$ denote the number of $d$-cells defined by the hyperplane arrangement of pairwise comparisons between these objects (i.e. $Q(n,d) = |\Sigma_{n,d}|$).  $Q(n,d)$ satisfies the recursion
\begin{align}
Q(n,d) = Q(n-1,d) + (n-1) Q(n-1,d-1) \ , \ \mbox{where $Q(1,d) =1$ and $Q(n,0) = 1$}.\label{counteq}
\end{align}
\vspace{-15pt}\end{lemma}
In the hyperplane arrangement induced by the $n$ objects in $d$ dimensions, each hyperplane is intersected by every other and is partitioned into $Q(n-1,d-1)$ subsets or $(d-1)$-cells. The recursion, above, arises by considering the addition of one object at a time. {Using this lemma in a straightforward fashion,
we prove the following corollary in Appendix~\ref{countCorProof}.}
\begin{corollary} \label{countCor} 
Assume \textbf{A1-2}. There exist positive real numbers $k_1$ and $k_2$ such that
\begin{align*}
k_1 \, \frac{n^{2d}}{2^{d}d!} < Q(n,d) < k_2 \, \frac{n^{2d}}{2^{d}d!}
\end{align*}
for $n > d+1$. If $n \leq d+1$ then $Q(n,d)=n!$.  For $n$ sufficiently large, $k_1 = 1$ and $k_2=2$ suffice.  
\end{corollary}

\subsection{Lower bounds on query complexity} \label{LBQC}
Since the  cardinality of the set of possible rankings is  $|\Sigma_{n,d}| = Q(n,d)$, we have a simple lower bound on the number of queries needed to determine the ranking.
\begin{theorem} \label{LB}
Assume \textbf{A1-2}. To reconstruct an arbitrary ranking $\sigma \in \Sigma_{n,d}$ any algorithm will require at least $\log_2 | \Sigma_{n,d} | = \Theta(2d \log_2 n)$ pairwise comparisons.
\vspace{-5pt}\end{theorem}
\begin{proof}
By Corollary~\ref{countCor} $|\Sigma_{n,d}|=\Theta(n^{2d})$, and so at least $2d \log n$ bits are needed to specify a ranking. Each pairwise comparison provides at most one bit. \vspace{-5pt}
\end{proof}  
If each query provides a full bit of information about the ranking, then we achieve this lower bound. For example, in the one-dimensional case ($d=1$) the objects can be ordered and binary search can be used to select pairwise comparison queries, achieving the lower bound. This is generally impossible in higher dimensions.  Even in two dimensions there are placements of the objects (still in general position) that produce $d$-cells in the partition induced by queries that have $n-1$ faces (i.e., bounded by $n-1$ hyperplanes) as shown in Appendix~\ref{nSides}. It follows that the worst case situation may require at least $n-1$ queries in dimensions $d\geq 2$. In light of this, we conclude that worst case bounds may be overly pessimistic indications of the typical situation, and  so we instead consider the average case performance introduced in Section~\ref{ps}.

\subsection{Inefficiency of random queries}
The geometrical representation of the ranking problem reveals that randomly choosing pairwise comparison queries is inefficient relative to the lower bound above. To see this, suppose $m$ queries were chosen uniformly at random from the possible $\binom{n}{2}$. The answers to $m$ queries narrows the set of possible rankings to a $d$-cell in $\R^d$.  This $d$-cell may consist of one or more of the $d$-cells in the partition induced by all queries.  If it contains more than one of the partition cells, then the underlying ranking is ambiguous. 
\begin{theorem}
Assume \textbf{A1-2}. Let $N=\binom{n}{2}$. Suppose $m$ pairwise comparison are chosen uniformly at random without replacement from the possible $\binom{n}{2}$. Then for all positive integers $N \geq m \geq d$ the probability that the $m$ queries yield a unique ranking is ${ \binom{m}{d} }/{ \binom{N}{d}} \leq( \frac{em}{N})^d$. 
\vspace{-5pt}\end{theorem}
\noindent{\em Proof.}
No fewer than $d$ hyperplanes bound each $d$-cell in the partition of $\R^d$ induced by all possible queries. The probability of selecting $d$ specific queries in a random draw of $m$ is equal to
\begin{align*}
{\binom{N-d}{m-d}}\bigg/{\binom{N}{m}} ={ \binom{m}{d} }\bigg/{ \binom{N}{d}} \leq \frac{m^d}{d!} \frac{d^d}{N^d} \leq \bigg( \frac{m}{N} \bigg)^d \frac{d^d}{d!} \leq  \bigg( \frac{em}{N} \bigg)^d. & \hfill & \Box
\end{align*}
Note that ${ \binom{m}{d} }/{ \binom{N}{d}} < 1/2$ unless  $m=\Omega(n^2)$. Therefore, if the queries are randomly chosen, then we will need to ask almost all queries to guarantee that the inferred ranking is probably correct.

\section{Analysis of sequential algorithm for query selection}
Now consider the basic sequential process of the algorithm in Figure~\ref{fig:alg}. Suppose we have ranked $k-1$ of the $n$ objects. Call these objects $1$ through $k-1$. This places the reference $r_{\sigma}$ within a $d$-cell (defined by the labels of the comparison queries between objects $1,\dots,k-1$). Call this $d$-cell $C_{k-1}$. Now suppose we pick another object at random and call it object $k$.  A comparison query between object $k$ and one of objects $1,\dots,k-1$ can only be informative (i.e., ambiguous) if the associated hyperplane intersects this $d$-cell $C_{k-1}$ (see Figure~\ref{ambiguous}). If $k$ is significantly larger than $d$, then it turns out that the cell $C_{k-1}$ is probably quite small and the probability that one of the queries intersects $C_{k-1}$ is very small; in fact the probability is on the order of $1/k^2$.

\subsection{Hyperplane-point duality}
Consider a hyperplane $\mathbf{h} = (h_0,h_1,\dots,h_d)$ with $(d+1)$ parameters in $\R^d$ and a point $\mathbf{p}=(p_1,\dots,p_d) \in \R^d$ that does not lie on the hyperplane. Checking which halfspace $\mathbf{p}$ falls in, i.e., $h_1 p_1 + h_2 p_2 + \dots +h_d p_d + h_0 \gtrless 0$,
has a dual interpretation: $\mathbf{h}$ is a point in $\R^{d+1}$ and $\mathbf{p}$ is a hyperplane in $\R^{d+1}$ passing through the origin (i.e., with $d$ free parameters). 

{Recall that each possible ranking can be represented by a reference point $r_{\sigma}  \in \R^d$. Our problem is to determine the ranking, or equivalently the vector of responses to the $\binom{n}{2}$ queries represented by hyperplanes in $\R^d$. Using the above observation, we see that our problem is equivalent to finding a labeling over $\binom{n}{2}$ points in $\R^{d+1}$ with as few queries as possible. We will refer to this alternative representation as the dual and the former as the primal.}

\subsection{Characterization of an ambiguous query} \label{ambigDef}
The characterization of an ambiguous query has interpretations in both the primal and dual spaces. We will now describe the interpretation in the dual which will be critical to our analysis of the sequential algorithm of Figure~\ref{fig:alg}.

\begin{definition} \cite{cover}
Let $S$ be a finite subset of $\R^d$ and let $S^+ \subset S$ be points labeled $+1$ and $S^- = S \setminus S^+$ be the points labeled $-1$ and let $x$ be any other point except the origin. If there exists two homogeneous linear separators of $S^+$ and $S^-$ that assign different labels to the point $x$,
then the label of $x$ is said to be {\em ambiguous with respect to }S.
\end{definition}

\begin{lemma}  \cite[Lemma 1]{cover} \label{coverSubspace}
The label of $x$ is ambiguous with respect to $S$ if and only if $S^+$ and $S^-$ are homogeneously linearly separable by a $(d-1)$-dimensional subspace containing $x$.\vspace{-5pt}\end{lemma} 
Let us consider the implications of this lemma to our scenario. Assume that we have labels for all the pairwise comparisons of $k-1$ objects. Next consider a new object called object $k$.  In the dual, the pairwise comparison between object $k$ and
object $i$, for some $i\in\{1,\dots,k-1\}$, is ambiguous if and only if there exists a hyperplane that still separates the original points and also passes through this new point. In the primal, this separating hyperplane corresponds to a point lying on the hyperplane defined by the associated pairwise comparison. 

\subsection{The probability that a query is ambiguous} \label{probAmbig}
An essential component of the sequential algorithm of Figure~\ref{fig:alg} is the initial random order of the objects; every sequence in which it could consider objects is equally probable. This allows us to state a nontrivial fact about the partial rankings of the first $k$ objects observed in this sequence.

\begin{lemma} \label{equiprobable}
Assume \textbf{A1-2} and $\sigma \sim \U$. Consider the subset $S \subset \Theta$ with $|S|=k$ that is randomly selected from $\Theta$ such that all $\binom{n}{k}$ subsets are equally probable. If $\Sigma_{k,d}$ denotes the set of possible rankings of these $k$ objects then every $\sigma\in \Sigma_{k,d}$ is equally probable.\vspace{-10pt}\end{lemma}\begin{proof}
Let a $k$-partition denote the partition of $\R^d$ into $Q(k,d)$ $d$-cells induced by $k$ objects for $1 \leq k \leq n$. In the $n$-partition, each $d$-cell is weighted uniformly and is equal to $1/Q(n,d)$. If we uniformly at random select $k$ objects from the possible $n$ and consider the $k$-partition, each $d$-cell in the $k$-partition will contain one or more $d$-cells of the $n$-partition. If we select one of these $d$-cells from the $k$-partition, on average there will be $Q(n,d)/Q(k,d)$ $d$-cells from the $n$-partition contained in this cell. Therefore the probability mass in each $d$-cell of the $k$-partition is equal to the number of cells from the $n$-partition in this cell multiplied by the probability of each of those cells from the $n$-partition: $Q(n,d)/Q(k,d) \times 1/Q(n,d) = 1/Q(k,d)$, and $|\Sigma_{k,d}|=Q(k,d)$.
\end{proof}
As described above, for $1 \leq i \leq  k$ some of the pairwise comparisons $q_{i,k+1}$ may be ambiguous. The algorithm chooses a random sequence of the $n$ objects in its initialization and does not use the labels of $q_{1,k+1},\dots,q_{j-1,k+1},q_{j+1,k+1},\dots,q_{k,k+1}$ to make a determination of whether or not $q_{j,k+1}$ is ambiguous. It follows that the events of requesting the label of $q_{i,k+1}$ for $i=1,2,\dots,k$ are independent and identically distributed (conditionally on the results of queries from previous steps). Therefore it makes sense to talk about the probability of requesting any one of them.

\begin{lemma} \label{uniformLemma}
Assume \textbf{A1-2} and $\sigma \sim \U$. Let $A(k,d,\U)$ denote the probability of the event that the pairwise comparison $q_{i,k+1}$ is ambiguous for $i=1,2,\dots,k$. Then there exists a positive, real number constant $a$ independent of $k$ such that for $k\geq2d$,  $A(k,d,\U) \leq a \frac{2d}{k^2}$. 
\end{lemma}
%
\begin{proof}
By Lemma~\ref{coverSubspace}, a point in the dual (pairwise comparison) is ambiguous if and only if there exists a separating hyperplane that passes through this point. This implies that the hyperplane representation of the pairwise comparison in the primal intersects the cell containing $r_{\sigma}$ (see Figure~\ref{ambiguous} for an illustration of this concept). Consider the partition of $\R^d$ generated by the hyperplanes corresponding to pairwise comparisons between objects $1,\dots,k$.  Let $P(k,d)$ denote the number of $d$-cells in this partition that are  intersected by a hyperplane corresponding to one of the queries $q_{i,k+1}$, $i\in\{1,\dots,k\}$. {Then it is not difficult to show that $P(k,d)$ is bounded above by a constant independent of $n$ and $k$ times $\textstyle{ \frac{k^{2(d-1)}}{2^{d-1}(d-1)!}}$ (see Appendix~\ref{uniformLemmaProof}).} By Lemma~\ref{equiprobable}, every $d$-cell in the partition induced by the $k$ objects corresponds to an  equally probable ranking of those objects. Therefore, the probability that a query is ambiguous is the number of cells intersected by the corresponding hyperplane divided by the total number of $d$-cells, and therefore $A(k,d,\U)=\frac{ P(k,d) }{ Q(k,d) }$. The result follows immediately from the bounds on $P(k,d)$ and Corollary \ref{countCor}. \vspace{-5pt}
\end{proof}

Because the individual events of requesting each query are conditionally independent, the total number of queries requested by the algorithm is just $M_n = \textstyle\sum_{k=1}^{n-1} \sum_{i=1}^{k}  \mathbf{1} \{\mbox{Request } q_{i,k+1} \}$.  {Using the results above, it straightforward to prove the main theorem below (see Appendix~\ref{simpleTheoremProof}).}
\begin{theorem} \label{simpleTheorem} 
Assume \textbf{A1-2} and $\sigma \sim \U$. Let the random variable $M_n$ denote the number of pairwise comparisons that are requested in the algorithm of Figure~\ref{fig:alg}, then 
\begin{align*}
\E_\U[M_n] \leq \lceil 2d a \rceil \log_2 n.
\end{align*}
Furthermore, if $\sigma \sim \pi$ and $\max_{\sigma\in\Sigma_{n,d}} \pi_\sigma \leq c | \Sigma_{n,d}|^{-1}$ for some $c>0$, then $\E_\pi[M_n] \leq c \, \E_\U[M_n]$. 
\vspace{-5pt}\end{theorem}

\section{Robust sequential algorithm for query selection} \label{robust}
We now extend the algorithm of Figure~\ref{fig:alg} to situations in which the response to each query is only probably correct. If the correct label of a query $q_{i,j}$ is $y_{i,j}$, we denote the possibly incorrect response by $Y_{i,j}$. Let the probability that $Y_{i,j}=y_{i,j}$ be equal to $1-p$, $p<1/2$. The robust algorithm operates in the same fashion as the algorithm in Figure~\ref{fig:alg}, with the exception that when an ambiguous query is encountered several (equivalent) queries are made and a decision is based on the majority vote. We will now judge performance based on two metrics: (i) how many queries are requested and (ii) how accurate the estimated ranking is with respect to the true ranking before it was corrupted. For any two rankings $\sigma$, $\widehat{\sigma}$ we adopt the popular Kendell-Tau distance \cite{marden}
\begin{align} \label{kTau}
d_\tau(\sigma,\widehat{\sigma}) = \sum_{(i,j): \sigma(i) < \sigma(j)} \1\{ \widehat{\sigma}(j) < \widehat{\sigma}(i) \}
\end{align}
where $\1$ is the indicator function. Clearly, $d_\tau(\sigma,\widehat{\sigma}) = d_\tau(\widehat{\sigma},\sigma)$ and $0 \leq d_\tau(\sigma,\widehat{\sigma}) \leq \binom{n}{2}$. For any ranking $\sigma \in \Sigma_{n,d}$ we wish to find an estimate $\widehat{\sigma}\in \Sigma_{n,d}$ that is close in terms of  $d_\tau(\sigma,\widehat{\sigma})$ without requesting too many pairwise comparisons. For convenience, we will some times report results in terms of the proportion $\epsilon$ of incorrect pairwise orderings such that $d_\tau(\sigma,\widehat{\sigma}) \leq \epsilon \binom{n}{2}$. Using the equivalence of the Kendell-Tau and Spearman's footrule distances (see \cite{diaconisEquiv}), if $d_\tau(\sigma,\widehat{\sigma}) \leq \epsilon \binom{n}{2}$ then each object in $\widehat{\sigma}$ is, on average, no more than $O(\epsilon n)$ positions away from its position in $\sigma$. Thus, the Kendell-Tau distance is an intuitive measure of closeness between two rankings.

First consider the case in which each query can be repeated to obtain multiple independent responses (votes) for each comparison query.  This {\em random errors} model arises, for example, in social choice theory where the ``reference'' is a group of people, each casting a vote.
{The elementary proof of the next theorem is given in Appendix~\ref{iidNoiseProof}.}
\begin{theorem} \label{iidNoise} 
{Assume \textbf{A1-2} and $\textstyle{\sigma \sim \U}$ but that each response to the query $q_{i,j}$ is a realization of an i.i.d.\ Bernoulli random variable $Y_{i,j}$ with $P(Y_{i,j}\neq y_{i,j}) \leq p < 1/2$ for all distinct $i,j \in \{1,\dots,n\}$. If all ambiguous queries are decided by the majority vote of $R$ independent responses to each such query, then with probability greater than $1-2 n \log_2 (n) \exp(-\frac{1}{2}(1-2p)^2R)$ this procedure correctly identifies the correct ranking (i.e. $\epsilon = 0$) and requests no more than $O(R d \log n)$ queries on average.}
\vspace{-5pt}\end{theorem}
We can deduce from the above theorem that to exactly recover the true ranking under the stated conditions with probability $1-\delta$, one need only request $O\big(d (1-2p)^{-2} \log^2(n / \delta) \big)$ pairwise comparisons, on average.

In other situations,  if we ask the same query multiple times we may get the same, possibly incorrect, response each time. This {\em persistent errors} model is natural, for example, if the reference is a {\em single} human. Under this model, if two rankings differ by only a single pairwise comparison, then they cannot be distinguished with probability greater than $1-p$. So, in general,  exact recovery of the ranking cannot be guaranteed with high probability. The best we can hope for is to exactly recover a {\em partial ranking} of the objects (i.e. the ranking over a subset of the objects) or a ranking that is merely probably approximately correct in terms of the Kendell-Tau distance of $(\ref{kTau})$. We will first consider the task of exact recovery of a partial ranking of objects and then turn our attention to the recovery of an approximate ranking.  Henceforth, we will assume the errors are persistent.

\subsection{Robust sequential algorithm for persistent errors} \label{alg2Sec}
The robust query selection algorithm for persistent errors is presented in Figure~\ref{fig:alg2}. The key ingredient in the persistent errors setting is the design of a voting set for each ambiguous query encountered.  Suppose the query $q_{i,j}$ is ambiguous in the algorithm of Figure~\ref{fig:alg}.  In principle, a voting set could be constructed using objects ranked between $i$ and $j$.  If object $k$ is between $i$ and $j$, then note that $y_{i,j}=y_{i,k}=y_{k,j}$.  In practice, we cannot identify the subset of objects ranked between $i$ and $j$ exactly, but we can find a set that contains them.  For an ambiguous query $q_{i,j}$ define
\begin{align} \label{Tij}
T_{i,j} := \{ k \in \{1\,\dots,n\} \, : \, \mbox{ $q_{i,k}$, $q_{k,j}$, or both are ambiguous}\}.
\end{align}
Then $T_{i,j}$ contains all objects ranked between $i$ and $j$ (if $k$ is ranked between $i$ and $j$, and $q_{i,k}$ and $q_{k,j}$ are unambiguous, then so is $q_{i,j}$, a contradiction).  Furthermore, if the first $j-1$ objects ranked in the algorithm were selected uniformly at random (or initialized in a random order in the algorithm) Lemma~\ref{equiprobable} implies that each object in $T_{i,j}$ is ranked between $i$ and $j$ with probability at least $\textstyle{1/3}$ due to the uniform distribution over the rankings $\Sigma_{n,d}$ (see Appendix~\ref{robustThmProof} for an explanation). $T_{i,j}$ will be our voting set. If we follow the sequential procedure of the algorithm of Figure~\ref{fig:alg2}, the first query encountered, call it $q_{1,2}$, will be ambiguous and $T_{1,2}$ will contain all the other $n-2$ objects. However, at some point for some query $q_{i,j}$ it will become probable that the objects $i$ and $j$ are closely ranked. In that case, $T_{i,j}$ may be rather small, and so it is not always possible to find a sufficiently large voting set to accurately determine $y_{i,j}$.  Therefore, we must specify a size-threshold $R\geq 0$.  If the size of $T_{i,j}$ is at least $R$, then we draw $R$ indices from $T_{i,j}$ uniformly at random without replacement, call this set $\{ t_l \}_{l=1}^R$, and decide the label for $q_{i,j}$ by voting over the responses to $\{q_{i,k},q_{k,j}: k\in \{ t_l \}_{l=1}^R \}$; otherwise we pass over object $j$ and move on to the next object in the list. Given that $|T_{i,j}| \geq R$ the label of $q_{i,j}$ is determined by:
\begin{align} \label{voteRule}
0 \, \, \mathop{\gtreqless}_{j \prec i}^{i \prec j} \, \,  \sum_{k \in \{ t_l \}_{l=1}^R}  \1 \{ Y_{i,k}=1 \wedge Y_{k,j} =1 \} - \1 \{ Y_{i,k}=0 \wedge Y_{k,j}=0\}.
\end{align}
In the next section we will analyze this algorithm and show that it enjoys a very favorable query complexity while also admitting a probably approximately correct ranking.

\begin{figure}%
\centering
\parbox{3.8in}{%
\fbox{\parbox[b]{3.7in}{{\underline{\bf Robust Query Selection Algorithm}}  \\ \vspace{-.1in} \\
input: $n$ objects in $\R^d$, $R \geq 0$\\
initialize: objects $\Theta = \{\theta_1,\dots,\theta_n\}$ in uniformly random order, $\Theta' = \Theta$ \\ \vspace{-.1in} \\
for j=2,\dots,n \\
\mbox{ \   } for i=1,\dots,j-1 \\
\mbox{ \ \ \ \ \bf if} $q_{i,j}$ is {\em ambiguous}, \\
\mbox{ \ \ \ \ \ } $T_{i,j} := \{ k \in \{1\,\dots,n\} \, : \, \mbox{ $q_{i,k}$, $q_{k,j}$, or both are ambiguous}\}$ \\
\mbox{ \ \ \ \ \ \ \bf if} $|T_{i,j}| \geq R$ \\
\mbox{ \ \ \ \ \ \ \ } $\{t_l \}_{l=1}^R \stackrel{i.i.d.}{\sim} \mbox{uniform}(T_{i,j})$.  \\
\mbox{ \ \ \ \ \ \ \ } request $Y_{i,k}$, $Y_{k,j}$ for all $k \in \{t_l \}_{l=1}^R$ \\
\mbox{ \ \ \ \ \ \ \ } decide label of $q_{i,j}$ with $(\ref{voteRule})$\\
\mbox{ \ \ \ \ \ \ \bf else} \\
\mbox{ \ \ \ \ \ \ \ } $\Theta' \leftarrow \Theta' \setminus \theta_j$ \,,\, \, $j \leftarrow j+1$ \\
\mbox{ \ \ \ \ \bf else} \\
\mbox{ \ \ \ \ \ } impute $q_{i,j}$'s label from previously labeled queries. \\ \vspace{-.1in} \\
output: ranking over objects in $\Theta'$
}}
\caption{Robust sequential algorithm for selecting queries of Section~\ref{alg2Sec}.  See Figure~\ref{ambiguous} and Section~\ref{ambigDef} for the definition of an ambiguous query. }
\label{fig:alg2}
}%
\end{figure}%

\subsection{Analysis of the robust sequential algorithm}  \label{partialRankingRecov}

Consider the robust algorithm in Figure~\ref{fig:alg2}. At the end of the process, some objects that were passed over may then be unambiguously ranked (based on queries made after they were passed over) or they can be ranked without voting (and without guarantees). As mentioned in Section~\ref{alg2Sec}, if the first $j-1$ objects ranked in the algorithm of Figure~\ref{fig:alg2} were chosen uniformly at random from the full set (i.e., none of the first $j-1$ objects were passed over) then there is at least a one in three chance each object in $T_{i,j}$ for some ambiguous query $q_{i,j}$ is ranked between $i$ and $j$. With this in mind, we have the following theorem, proved in Appendix~\ref{robustThmProof}. 

\begin{theorem}  \label{robustThm}
Assume \textbf{A1-2}, $\sigma \sim \U$, and $\textstyle{P(Y_{i,j}\neq y_{i,j}) = p}$. For every set $T_{i,j}$ constructed in the algorithm of Figure~\ref{fig:alg2}, assume that an object selected uniformly at random from $T_{i,j}$ is ranked between $\theta_i$ and $\theta_j$ with probability at least $1/3$. Then for any size-threshold $R\geq 1$, with probability greater than $1-2\textstyle{n \log_2 (n) \exp \big( -\frac{2}{9} (1-2p)^2 R \big)}$ the algorithm correctly ranks at least $n / (2R+1)$ objects and requests no more than $O(R d \log n) $ queries on average. 
\end{theorem} 

Note that before the algorithm skips over an object for the first time, all objects that are ranked at such an intermediate stage are a subset chosen uniformly at random from the full set of objects, due to the initial randomization. Therefore, if $T_{i,j}$ is a voting set in this stage, an object selected uniformly at random from $T_{i,j}$ is ranked between $\theta_i$ and $\theta_j$ with probability at least $1/3$, per Lemma~\ref{equiprobable}.  After one or more objects are passed over, however, the distribution is no longer necessarily uniform due to this action, and so the assumption of the theorem above may not hold.  The procedure of the algorithm is still reasonable, but it is difficult to give guarantees on performance without the assumption. Nevertheless, this discussion leads us to wonder how many objects the algorithm will rank before it skips over its first object.  The next lemma is proved in Appendix~\ref{seqBoundProof}.

\begin{lemma} \label{seqBound}
Consider a ranking of $n$ objects and suppose objects are drawn sequentially, chosen uniformly at random without replacement. If $M$ is the largest integer such that $M$ objects are drawn before any object is within $R$ positions of another one in the ranking, then $M \geq \sqrt{\frac{n/R}{6 \log(2)}}$ with probability at least $\frac{1}{6 \log(2)}  \big( e^{-(\sqrt{6 \log(2)   R/n} +1)^2 /2}  - 2^{-n / (3R)} \big)$. As $n/R \rightarrow \infty$,  ${P(M \geq \sqrt{\frac{n/R}{6 \log(2)}}) \rightarrow \frac{1}{6 \sqrt{e} \log(2)}}$.
\end{lemma}

Lemma~\ref{seqBound} characterizes how many objects the robust algorithm will rank before it passes over its first object because if there are at least $R$ objects between every pair of the first $M$ objects, then $T_{i,j} \geq R$ for all distinct $i,j \in \{1,\dots,M\}$ and none of the first $M$ objects will be passed over.  We can conclude from Lemma~\ref{seqBound} and Theorem~\ref{robustThm} that with constant probability (with respect to the initial ordering of the objects and the randomness of the voting), the algorithm of Figure~\ref{fig:alg2} exactly recovers a partial ranking of at least ${\Omega(\sqrt{ (1-2p)^2 n  / \log n})}$ objects by requesting just ${O\big(d  (1-2p)^{-2}\log^2n  \big)}$ pairwise comparisons, on average, with respect to all the rankings in $\Sigma_{n,d}$. If we repeat the algorithm with different initializations of the objects each time, we can boost this constant probability to an arbitrarily high probability (recall that the responses to queries will not change over the repetitions). Note, however, that the correctness of the partial ranking does not indicate how approximately correct the remaining rankings will be. If the algorithm of Figure~\ref{fig:alg2} ranks $m$ objects before skipping over its first, then the next lemma quantifies how accurate an estimated ranking is in terms of Kendel-Tau distance, given that it is some ranking in $\Sigma_{n,d}$ that is consistent with the probably  correct partial ranking of the first $m$ objects (the output ranking of the algorithm may contain more than $m$ objects but we make no guarantees about these additional objects). The proof is available in Appendix~\ref{goodNewsProof}.
\begin{lemma} \label{goodNews}
Assume \textbf{A1-2} and $\sigma \sim \U$. Suppose we select $1 \leq m < n$ objects uniformly at random from the $n$ and correctly rank them amongst themselves. If $\widehat{\sigma}$ is any ranking in $\Sigma_{n,d}$ that is consistent with all the known pairwise comparisons between the $m$ objects, then $\E[ d_\tau(\sigma,\widehat{\sigma}) ] = O( d / m^2 ) \binom{n}{2}$, where the expectation is with respect to the random selection of objects and the distribution of the rankings $\U$.
\end{lemma} 

Combining Lemmas \ref{seqBound} and \ref{goodNews} in a straightforward way, we have the following theorem.
\begin{theorem} \label{approxRanking}
Assume \textbf{A1-2}, $\sigma \sim \U$, and $\textstyle{P(Y_{i,j}\neq y_{i,j}) = p}$. If $R = \Theta((1-2p)^{-2} \log n)$ and $\widehat{\sigma}$ is any ranking in $\Sigma_{n,d}$ that is consistent with all  known pairwise comparisons between the subset of objects ranked in the output of the algorithm of Figure~\ref{fig:alg2}, then with constant probability ${\E[ d_\tau(\sigma,\widehat{\sigma}) ] = {O}(d  (1-2p)^{-2} \log(n ) /n)  \binom{n}{2}}$ and no more than ${O}(d  (1-2p)^{-2} \log^2 (  n ) )$ pairwise comparisons are requested, on average.
\end{theorem}
If we repeat the algorithm with different initializations of the objects until a sufficient number of objects are ranked before an object is passed over, we can boost this constant probability to an arbitrarily high probability. However, in practice, we recommend running the algorithm just once to completion since we do not believe passing over an object early on greatly affects performance. 

\section{Empirical results} \label{empResults}
In this section we present empirical results for both the error-free algorithm of Figure~\ref{fig:alg} and the robust algorithm of Figure~\ref{fig:alg2}. For the error-free algorithm, $n=100$ points, representing the objects to be ranked, were uniformly at random simulated from the unit hypercube $[0,1]^d$ for $d=1,10,20,\dots,100$. The reference was simulated from the same distribution. For each value of $d$ the experiment was repeated $25$ times using a new simulation of points and the reference. Because responses are error-free, exact identification of the ranking is guaranteed. The number of requested queries is plotted in Figure~\ref{simResults} with the lower bound of Theorem~\ref{LB} for reference. The number of requested queries never exceeds twice the lower bound which agrees with the result of Theorem~\ref{simpleTheorem}.

\begin{table}%
\vspace{-.15in}
\centering
\parbox{2.7in}{
\begin{center}
\includegraphics[scale=.44]{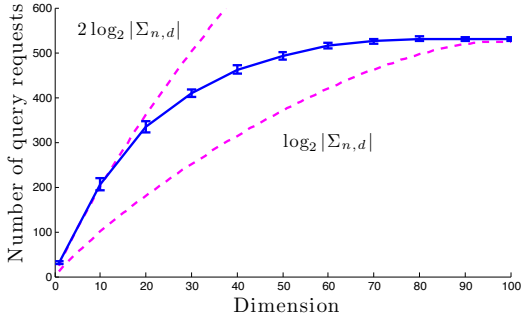} 
\captionof{figure}{Mean and standard deviation of requested queries (solid) in the error-free case for $n=100$; $\log_2 | \Sigma_{n,d}|$ is a lower bound (dashed).}
\label{simResults}
\end{center}
}%
\qquad
\begin{minipage}{2.5in}%
\vspace{-0.1in}
\caption{Statistics for the algorithm robust to persistent errors of Section~\ref{robust} with respect to all $\binom{n}{2}$ pairwise comparisons. Recall $y$ is the noisy response vector, $\tilde{y}$ is the embedding's solution, and $\hat{y}$ is the output of the robust algorithm.}
\label{robustResults}
\begin{center}

  \begin{tabular}{| b{2cm} @{} | c | c | c | }
\hline
\multicolumn{2}{|l|}{Dimension} & 2 & 3 \\
 \hline
\multirow{2}{2cm}{\% of queries requested} & mean & 14.5 &  18.5 \\ \cline{2-4}
  & std & 5.3 & 6 \\ \hline
\multirow{2}{2cm}{Average error} & ${d(y,\tilde{y})}$ & 0.23 & 0.21 \\ \cline{2-4}
 & ${d(y,\hat{y})}$ & 0.31 & 0.29 \\ \hline
  \end{tabular}
  
 \end{center}
\end{minipage}%
\end{table}%

{The robust algorithm of Figure~\ref{fig:alg2} was evaluated using a symmetric similarity matrix dataset available at \cite{simRepository} whose $(i,j)$th entry, denoted $s_{i,j}$, represents the human-judged similarity between audio signals $i$ and $j$ for all $i\neq j \in \{1,\dots,100\}$. If we consider the $k$th row of this matrix, we can rank the other signals with respect to their similarity to the $k$th signal; we define ${q_{i,j}^{(k)} :=  \{s_{k,i} > s_{k,j}\}}$ and $y_{i,j}^{(k)} := \1\{q_{i,j}^{(k)}\}$. Since the similarities were derived from human subjects, the derived labels may be erroneous.  Moreover, there is no possibility of repeating queries here and so the errors are persistent.} The analysis of this dataset in \cite{philips} suggests that the relationship between signals can be well approximated by an embedding in 2 or 3 dimensions. We used {\em non-metric multidimensional scaling} \cite{cox} to find an embedding of the signals: $\theta_1,\dots,\theta_{100}\in\R^d$ for $d=2$ and $3$.  {For each object $\theta_k$, we use the embedding to derive pairwise comparison labels between all other objects as follows: ${\tilde{y}_{i,j}^{(k)} :=  \1\{||\theta_k-\theta_i|| < ||\theta_k-\theta_j||\}}$, which can be considered as the best approximation to the labels $y^{(k)}_{i,j}$ (defined above) in this embedding.} The output of the robust sequential algorithm, which uses only a small fraction of the similarities, is denoted by $\hat{y}_{i,j}^{(k)}$. We set $R=15$ using Theorem~\ref{approxRanking} as a rough guide. Using the popular Kendell-Tau distance $d(y^{(k)},\hat{y}^{(k)})={\binom{n}{2}}^{-1} \sum_{i < j} \1\{y_{i,j}^{(k)} \neq \hat{y}_{i,j}^{(k)}\}$ \cite{marden} for each object $k$, we denote the average of this metric over all objects by ${d(y,\hat{y})}$ and report this statistic and the number of queries requested in Table~\ref{robustResults}. Because the average error of $\hat{y}$ is only $0.07$ higher than that of $\tilde{y}$, this suggests that the algorithm is doing almost as well as we could hope. Also, note that $2R \, 2d \log n / \binom{n}{2}$ is equal to $11.4\%$ and $17.1\%$ for $d=2$ and $3$, respectively, which agrees well with the {experimental} values.

\bibliography{activeRanking_extended}

\begin{thebibliography}{10}

\bibitem{knuth}
D.~Knuth.
\newblock {\em {The Art of Computer Programming, Volume 3: Sorting and
  Searching}}.
\newblock Addison-Wesley, 1998.

\bibitem{philips}
Scott Philips, James Pitton, and Les Atlas.
\newblock Perceptual feature identification for active sonar echoes.
\newblock In {\em {OCEANS} 2006}, 2006.

\bibitem{mcfee}
B.~McFee and G.~Lanckriet.
\newblock {Partial order embedding with multiple kernels}.
\newblock In {\em Proceedings of the 26th Annual International Conference on
  Machine Learning}, pages 721--728. ACM, 2009.

\bibitem{gormley}
I.~Gormley and T.~Murphy.
\newblock {A latent space model for rank data}.
\newblock {\em Statistical Network Analysis: Models, Issues, and New
  Directions}, pages 90--102, 2007.

\bibitem{cox}
M.A.A. Cox and T.F. Cox.
\newblock {Multidimensional scaling}.
\newblock {\em Handbook of data visualization}, pages 315--347, 2008.

\bibitem{traub}
J.F. Traub.
\newblock {\em Information-based complexity}.
\newblock John Wiley and Sons Ltd., 2003.

\bibitem{coombs}
C.H. Coombs.
\newblock {A theory of data}.
\newblock {\em Psychological review}, 67(3):143--159, 1960.

\bibitem{cover}
T.M. Cover.
\newblock {Geometrical and statistical properties of systems of linear
  inequalities with applications in pattern recognition}.
\newblock {\em IEEE transactions on electronic computers}, 14(3):326--334,
  1965.

\bibitem{dasgupta}
S.~Dasgupta, A.T. Kalai, and C.~Monteleoni.
\newblock {Analysis of perceptron-based active learning}.
\newblock {\em The Journal of Machine Learning Research}, 10:281--299, 2009.

\bibitem{hanneke}
S.~Hanneke.
\newblock {\em {Theoretical foundations of active learning}}.
\newblock PhD thesis, Citeseer, 2009.

\bibitem{hegedus}
Tibor Heged\"{u}s.
\newblock Generalized teaching dimensions and the query complexity of learning.
\newblock In {\em Proceedings of the eighth annual conference on Computational
  learning theory}, COLT '95, pages 108--117, New York, NY, USA, 1995. ACM.

\bibitem{freund}
Y.~Freund, R.~Iyer, R.E. Schapire, and Y.~Singer.
\newblock An efficient boosting algorithm for combining preferences.
\newblock {\em The Journal of Machine Learning Research}, 4:933--969, 2003.

\bibitem{burges}
C.~Burges, T.~Shaked, E.~Renshaw, A.~Lazier, M.~Deeds, N.~Hamilton, and
  G.~Hullender.
\newblock Learning to rank using gradient descent.
\newblock In {\em Proceedings of the 22nd international conference on Machine
  learning}, pages 89--96. ACM, 2005.

\bibitem{zheng}
Z.~Zheng, K.~Chen, G.~Sun, and H.~Zha.
\newblock A regression framework for learning ranking functions using relative
  relevance judgments.
\newblock In {\em Proceedings of the 30th annual international ACM SIGIR
  conference on Research and development in information retrieval}, pages
  287--294. ACM, 2007.

\bibitem{herbrich}
R.~Herbrich, T.~Graepel, and K.~Obermayer.
\newblock Support vector learning for ordinal regression.
\newblock In {\em Artificial Neural Networks, 1999. ICANN 99. Ninth
  International Conference on (Conf. Publ. No. 470)}, volume~1, pages 97--102.
  IET, 1999.

\bibitem{lu}
T.~Lu and C.~Boutilier.
\newblock Robust approximation and incremental elicitation in voting protocols.
\newblock {\em IJCAI-11, Barcelona}, 2011.

\bibitem{chu}
W.~Chu and Z.~Ghahramani.
\newblock Extensions of gaussian processes for ranking: semi-supervised and
  active learning.
\newblock {\em Learning to Rank}, page~29, 2005.

\bibitem{bennett}
J.F. Bennett and W.L. Hays.
\newblock {Multidimensional unfolding: Determining the dimensionality of ranked
  preference data}.
\newblock {\em Psychometrika}, 25(1):27--43, 1960.

\bibitem{marden}
J.I. Marden.
\newblock {\em Analyzing and modeling rank data}.
\newblock Chapman \& Hall/CRC, 1995.

\bibitem{diaconisEquiv}
P.~Diaconis and R.L. Graham.
\newblock Spearman's footrule as a measure of disarray.
\newblock {\em Journal of the Royal Statistical Society. Series B
  (Methodological)}, pages 262--268, 1977.

\bibitem{simRepository}
Similarity Learning. Aural~Sonar dataset.
  [http://idl.ee.washington.edu/SimilarityLearning]. University~of Washington
  Information Design~Lab, 2011.

\end{thebibliography}

\appendix
\section{Appendix}

\subsection{Computational complexity and implementation} \label{complexity}
The computational complexity of the algorithm in Figure~\ref{fig:alg} is determined by the complexity of testing whether a query is ambiguous or not and how many times we make this test. As written in Figure~\ref{fig:alg}, the test would be performed $O(n^2)$ times. But if binary sort is used instead of the brute-force linear search this can be reduced to $n \log_2 n$ and, in fact, this is implemented in our simulations and the proofs of the main results. The complexity of each test is polynomial in the number of queries requested because each one is a linear constraint. Because our results show that no more than $O(d \log n)$ queries are requested, the overall complexity is no greater than $O( n \ \text{poly}(d) \ \text{poly}(\log n) )$. 

\subsection{Proof of Corollary \ref{countCor}} \label{countCorProof}
\begin{proof}
For initial conditions given in Lemma \ref{counting}, if $d\ll n-1$ a simple manipulation of (\ref{counteq}) shows
\begin{eqnarray*}
Q(n,d) &=& 1 + \sum_{i=1}^{n-1} (n-i) Q(n-i,d-1)\\
&=& 1 + \sum_{i=1}^{n-1} i \, Q(i,d-1) \\
&=& 1 + \sum_{i=1}^{n-1} i  \bigg[ 1 + \sum_{j=1}^{i-1} j\, Q(j,d-2) \bigg] \\
&=& 1 + \Theta({n^2}/2) +  \sum_{i=1}^{n-1} \sum_{j=1}^{i-1} i \, j\,  \bigg[ 1 + \sum_{k=1}^{j-1} k\, Q(k,d-3) \bigg]  \\
&=& 1 + \Theta(n^2/2) + \Theta(n^4/2/4) + \sum_{i=1}^{n-1} \sum_{j=1}^{i-1} \sum_{k=1}^{j-1} i\, j\, k\,  \bigg[ 1 + \sum_{l=1}^{k-1} l\, Q(l,d-4) \bigg] \\
&=& 1 + \Theta(n^2/2) + \dots + \Theta\bigg(\frac{n^{2d}}{2^{d}d!}\bigg).\\
\end{eqnarray*}
From simulations, this is very tight for large values of $n$. If $d \geq n-1$ then $Q(n,d)=n!$ because any permutation of $n$ objects can be embedded in $n-1$ dimensional space \cite{coombs}.
\end{proof}

\subsection{Construction of a $d$-cell with $n-1$ sides} \label{nSides}
Situations may arise in which $\Omega(n)$ queries must be requested to identify a ranking because the $d$-cell representing the ranking is bounded by $n-1$ hyperplanes (queries) and if they are not all requested, the ranking is ambiguous. We now show how to construct this pathological situation in $\R^2$. Let $\Theta$ be a collection of $n$ points in $\R^2$ where each $\theta \in \Theta$ satisfies $\theta_1^2 = \theta_2$ and $\theta_1 \in [0,1]$ where $\theta_i$ denotes the $i$th dimension of $\theta$ ($i \in \{1,2\}$). Then there exists a $2$-cell in the hyperplane arrangement induced by the queries that has $n-1$ sides. This follows because the slope of the parabola keeps increasing with $\theta_1$ making at least one query associated with $(n-1)$ $\theta$'s bisect the lower-left, unbounded $2$-cell. This can be observed in Figure~\ref{pathological}. Obviously, a similar arrangement could be constructed for all $d\geq 2$. 

\begin{figure}[htbp]
\begin{center}
\includegraphics[scale=.54]{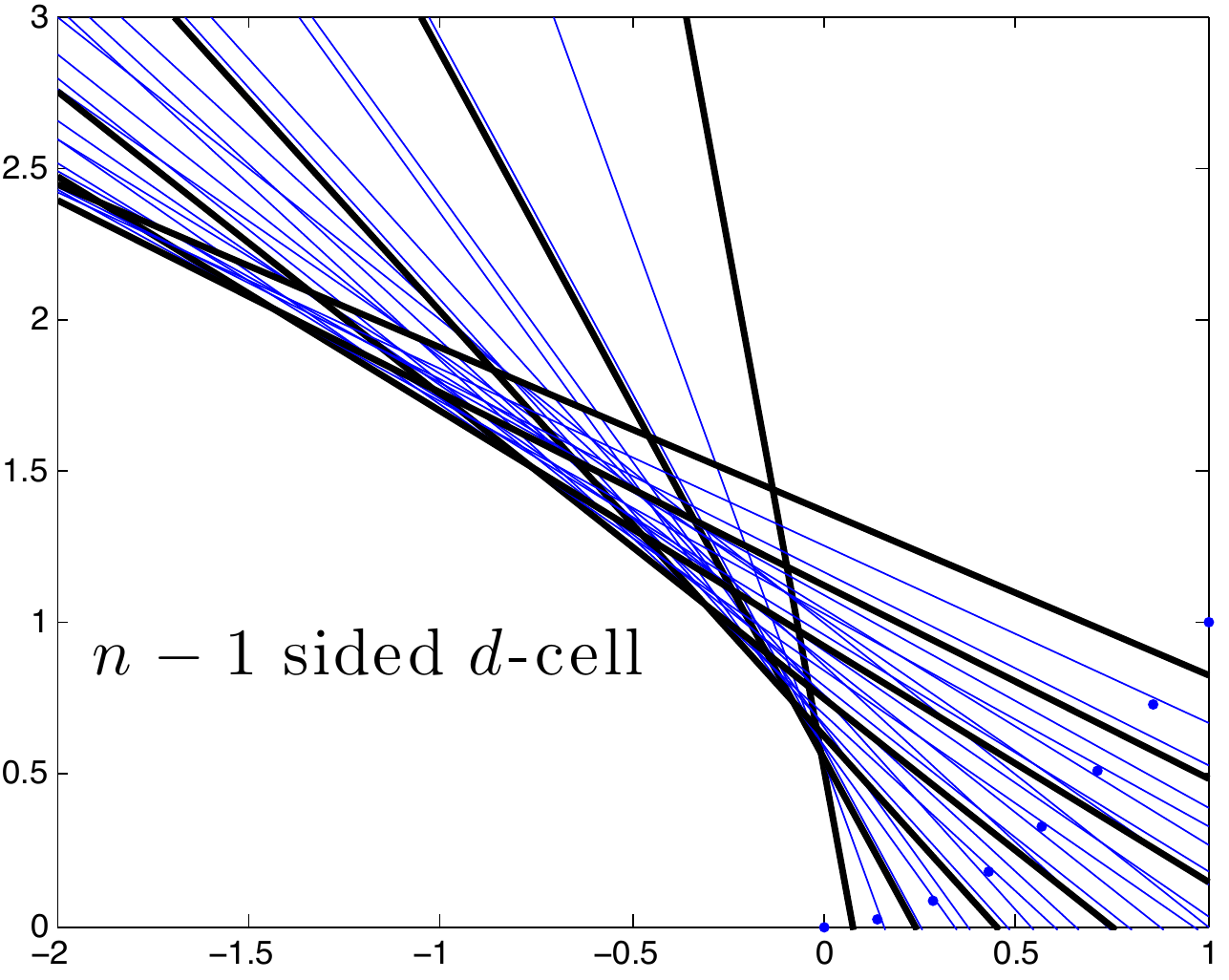} 
\caption{The points $\Theta$ representing the objects are dots on the right, the lines are the queries, and the black, bold lines are the queries bounding the $n-1$ sided $2$-cell.}
\label{pathological}
\end{center}
\end{figure}

\subsection{Proof of Lemma \ref{uniformLemma}} \label{uniformLemmaProof}
\begin{proof}
Here we prove an upper bound on $P(k,d)$. $P(k,d)$ is equal to the number of $d$-cells in the partition induced by objects $1,\dots,k$  that are intersected by a hyperplane corresponding to a pairwise comparison query between object $k+1$ and object $i$, $i\in\{1,\dots,k\}$. This new hyperplane is intersected by all the $\binom{k}{2}$  hyperplanes in the partition.  These intersections partition the new hyperplane into a number of $(d-1)$-cells. Because the $(k+1)$st object is in general position with respect to objects $1,\dots,k$, the intersecting hyperplanes will not intersect the hyperplane in any special or non-general way. That is to say, the number of $(d-1)$-cells this hyperplane is partitioned into is the same number that would occur if the hyperplane were intersected by $\binom{k}{2}$ hyperplanes in general position. Let $K= \binom{k}{2}$ for ease of notation. It follows then from \cite[Theorem 3]{cover} that \begin{eqnarray*}
P(k,d) &=&  \sum_{i=0}^{d-1} \binom{K}{i} \ 
= \  \sum_{i=0}^{d-1}  O \bigg( \frac{K^i}{i!} \bigg) \ = \  \sum_{k=0}^{d-1}  O \bigg( \frac{k^{2i}}{2^i i!} \bigg) 
\ = \   O \bigg( \frac{k^{2(d-1)}}{2^{d-1} (d-1)!}  \bigg).
\end{eqnarray*}
\end{proof}

\subsection{Proof of Theorem \ref{simpleTheorem}} \label{simpleTheoremProof}
\begin{proof}
Let $B_{k+1}$ denote the total number of pairwise comparisons requested of the $(k+1)$st object; i.e., number of ambiguous queries in the set $q_{i,k+1}$, $i=1,\dots,k$. Because the individual events of requesting these are conditionally independent (see Section~\ref{probAmbig}), it follows that each $B_{k+1}$ is an independent binomial random variable with parameters $A(k,d,\U)$ and $k$. The total number of queries requested by the algorithm is
\begin{align} \label{binBern}
M_n = \sum_{k=1}^{n-1} \sum_{i=1}^{k}  \mathbf{1} \{\mbox{Request } q_{i,k+1} \} =\sum_{k=1}^{n-1} B_{k+1} \ .
 \end{align}
Because Lemma~\ref{uniformLemma} is only relevant for sufficiently large $k$, we assume that none of the pairwise comparisons are ambiguous when $k\leq 2d a$. Recall from Section~\ref{complexity} that binary sort is implemented so for these first $\lceil 2d a \rceil$ objects, at most $\lceil 2d a \rceil \log_2(\lceil 2d a \rceil)$ queries are requested. For $k > 2d a$ the number of requested queries to the $k$th object is upper bounded by the number of ambiguous queries of the $k$th object. Then using the known mean and variance formulas for the binomial distribution
\begin{align*}
\E_\U \big[ M_n \big] &=  \sum_{k=1}^{n-1} \E_\U\big[ B_{k+1} \big] \\
&\leq \sum_{k=2}^{\lceil 2d a \rceil} B_{k+1} + \sum_{k={\lceil 2d a \rceil}+1}^{n-1}   \frac{2d a}{k}\\
&\leq \lceil 2d a \rceil \log_2\lceil 2d a \rceil + 2 d a \log \big( n / \lceil 2d a \rceil \big)\\
& \leq \lceil 2d a \rceil \log_2 n
\end{align*}

We now consider the case for a general distribution $\pi$. Enumerate the rankings of $\Sigma_{n,d}$. Let $N_i$ denote the (random) number of requested queries needed by the algorithm to reconstruct the $i$th ranking. Note that the randomness of $N_i$ is only due to the randomization of the algorithm. Let $\pi_i$ denote the probability it assigns to the $i$th ranking as in Definition~\ref{distributionDef}.  Then
\begin{eqnarray} 
\E_\pi [M_n] & = &\sum_{i=1}^{Q(n,d)} \pi_i \, \E[N_i] \label{weightedE}.
\end{eqnarray}
Assume that the distribution over rankings is bounded above such that no ranking is overwhelmingly probable. Specifically, assume that the probability of any one ranking is upper bounded by $c / Q(n,d)$ for some constant $c>1$ that is independent of $n$. Under this bounded distribution assumption, $\E_\pi[M_n]$ is maximized by placing probability $c / Q(n,d)$ on the $k:=Q(n,d)/c$ cells for which $\E[N_i]$ is largest (we will assume $k$ is an integer, but it is straightforward to extend the following argument to the general case).  Since the mass on these cells is equal, without loss of generality we may assume that $\E[N_i] = \mu$, a common value on each, and we have $\E_\pi[M_n]=\mu$.  For the remaining $Q(n,d)-k$ cells we know that $\E[N_i]\geq d$, since each cell is bounded by at least $d$ hyperplanes/queries.  Under these conditions, we can relate $\E_\pi[M_n]$ to $\E_\U[M_n]$ as follows.  First observe that
\begin{eqnarray*}
\E_\U[M_n] & = & \frac{1}{Q(n,d)}\sum_{i=1}^{Q(n,d)} \E[N_i] \
\geq \ \frac{k}{Q(n,d)} \mu + d\, \frac{Q(n,d)-k}{Q(n,d)} \ ,
\end{eqnarray*}
which implies
$$\textstyle{\E_\pi[M_n] \ = \ \mu \ \leq \ \frac{Q(n,d)}{k}\left(\E_\U[M_n]-d\frac{Q(n,d)-k}{Q(n,d)}\right) \ = \ c \left(\E_\U[M_n]-d\frac{Q(n,d)-k}{Q(n,d)}\right) \ \leq \ c \, \E_\U[M_n] \  .} $$
In words, the non-uniformity constant $c>1$ scales the expected number of queries. Under $\textbf{A1-2}$, for large $n$ we have $\E_\pi[M_n] = O(c \, d\log n)$.
\end{proof}

\subsection{Proof of Theorem \ref{iidNoise}} \label{iidNoiseProof}
\begin{proof}
Suppose $q_{i,j}$ is ambiguous. Let $\hat{\alpha}$ be the frequency of $Y_{i,j} = 1$ after $R$ trials. Let $\E[\hat{\alpha}]=\alpha$. The majority vote decision is correct if $|\alpha - \hat{\alpha}| \leq 1/2 - p$. By Chernoff's bound, $\P( |  \alpha - \hat{\alpha} | \geq 1/2-p) \leq 2 \exp({-2(1/2-p)^2 R})$. The result follows from the union bound over the total number of queries considered: $n \log_2 n$ (See Appendix~\ref{complexity}).

\subsection{Proof of Theorem \ref{robustThm}} \label{robustThmProof}
Suppose $q_{i,j}$ is ambiguous. Let $S_{i,j}$ denote the subset of $\Theta$ such that $\theta_k \in S_{i,j}$ if it is ranked between objects $\theta_i$ and $\theta_j$ (i.e. $S_{i,j} = \{ \theta_k \in \Theta:  \theta_i \prec \theta_k \prec \theta_j$ or  $\theta_j \prec \theta_k \prec \theta_i \}$). Note that  $y_{i,j}=y_{i,k}=y_{k,j}$ if and only if $\theta_k \in S_{i,j}$. If we define $E_{i,j}^k = \1 \{ Y_{i,k}=1 \wedge Y_{k,j} =1 \} - \1 \{ Y_{i,k}=0 \wedge Y_{k,j}=0\}$, where $\1$ is the indicator function, then for any subset $T \subset \Theta$ such that $S_{i,j} \subset T$, the sign of the sum  $\textstyle\sum_{\theta_k \in T} E_{i,j}^k$ is a predictor of $y_{i,j}$. In fact, with respect to just the random errors, $\E\big[ \big| \textstyle\sum_{\theta_k \in T} E_{i,j}^k \big| \big] = |S_{i,j}| (1-2p)$. To see this, without loss of generality let $y_{i,j}=1$, then for  $\theta_k \in S_{i,j}$
\begin{align*}
\E[ E_{i,j}^k ] &= \E \big[ \1 \{ Y_{i,k}=1 \wedge Y_{k,j} =1 \} - \1 \{ Y_{i,k}=0 \wedge Y_{k,j}=0\} \big]\\
&= \P( Y_{i,k}=1 \wedge Y_{k,j} =1 ) - \P( Y_{i,k}=0 \wedge Y_{k,j}=0 )\\
&= (1-p)^2 - p^2\\
&= 1-2p.
\end{align*}
If $\theta_k \notin S_{i,j}$ then it can be shown by a similar calculation that $\E[ E_{i,j}^k ] = 0$.

To identify $S_{i,j}$ we use the fact that if $\theta_k \in S_{i,j}$ then $q_{i,k}$, $q_{j,k}$, or both are also ambiguous simply because otherwise $q_{i,j}$ would not have been ambiguous in the first place (Figure~\ref{threeObjects} may be a useful aid to see this). While the converse is false, Lemma~\ref{equiprobable} says that each of the six possible rankings of $\{ \theta_i,\theta_j,\theta_k \}$ are equally probable if they were uniformly at random chosen (thus partly justifying this explicit assumption in the theorem statement). It follows that if we define the subset $T_{i,j} \in \Theta$ to be those objects $\theta_k$ with the property that $q_{i,k}$, $q_{k,j}$, or both are ambiguous then the probability that $\theta_k \in S_{i,j}$ is at least $1/3$ if $\theta_k \subset T_{i,j}$. You can convince yourself of this using Figure~\ref{threeObjects}. Moreover, $\E\big[ \big| \textstyle\sum_{k \in T_{i,j}} E_{i,j}^k \big| \big] \geq |T_{i,j}| (1-2p) /3$ which implies the sign of the sum  $\textstyle\sum_{\theta_k \in T_{i,j}} E_{i,j}^k$ is a reliable predictor of $q_{i,j}$; just how reliable depends only on the size of $T_{i,j}$.

\begin{figure}[htbp]
\begin{center}
\includegraphics[scale=.44]{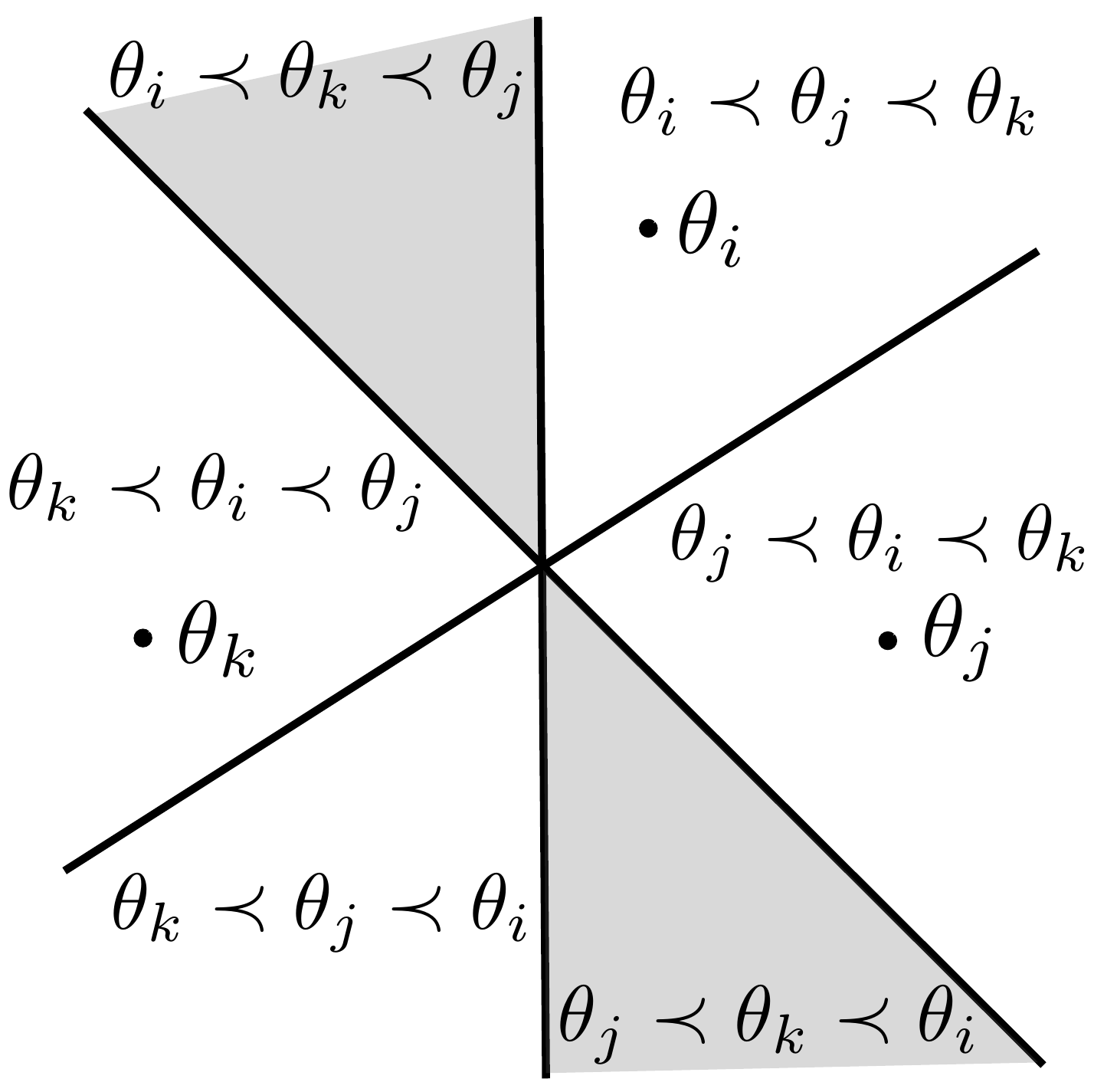} 
\caption{Let $q_{i,j}$ be ambiguous. Object $k$ will be informative to the majority vote of $y_{i,j}$ if the reference lies in the shaded region. There are six possible rankings and if $q_{i,k}$, $q_{k,j}$, or both are ambiguous then the probability that the reference is in the shaded region is at least $1/3$}
\label{threeObjects}
\end{center}
\end{figure}

Fix $R >0$. Suppose $q_{i,j}$ is ambiguous and assume without loss of generality that $y_{i,j}=1$. Given that $\E\big[  \textstyle\sum_{k \in T_{i,j}} E_{i,j}^k  \big] \geq |T_{i,j}| (1-2p) /3$ from above, it follows from Hoeffding's inequality that the probability that $ \textstyle\sum_{k \in T_{i,j}} E_{i,j}^k \leq 0$ is less than $\exp \big( -\frac{2}{9} (1-2p)^2 |T_{i,j}| \big)$. If only a subset of $T_{i,j}$ of size $R$ is used in the sum then $|T_{i,j}|$ is replaced by $R$ in the exponent. This test is only performed when $|T_{i,j}|>R$ and clearly no more times than the number of queries considered to rank $n$ objects in the full ranking: $n \log_2 n$. Thus, all decisions using this test are correct with probability at least $1-2 n \log_2 (n) \exp \big( -\frac{2}{9} (1-2p)^2 R \big)$. Only a subset of the $n$ objects will be ranked and of those, $2 R+1$ times more queries will be requested than in the error-free case (two queries per object in $T_{i,j}$). Thus the robust algorithm will request no more than $O(R d  \log n)$ queries on average.

To determine the number of objects that are in the partial ranking, let $\Theta' \subset \Theta$ denote the subset of objects that are ranked in the output partial ranking. Each $\theta_k \in \Theta'$ is associated with an index in the true full ranking and is denoted by $\sigma(\theta_k)$. That is, if $\sigma(\theta_k) = 5$ then it is ranked fifth in the full ranking but in the partial ranking could be ranked first, second, third, fourth, or fifth. Now imagine the real line with tick marks only at the integers $1,\dots,n$. For each $\theta_k \in \Theta'$ place an $R$-ball around each $\theta_k$ on these tick marks such that if $\sigma(\theta_k) = 5$ and $R = 3$ then $2,\dots,8$ are covered by the ball around $\sigma(\theta_k)$ and $1$ and $9,\dots,n$ are not. Then the union of the balls centered at the objects in $\Theta'$ cover $1,\dots,n$. If this were not true then there would be an object $\theta_j \notin \Theta'$ with $|S_{i,j}| > R$ for all $\theta_i \in \Theta'$. But $S_{i,j} \subset T_{i,j}$  implies  $|T_{i,j}|>R$ which implies $j \in \Theta'$, a contradiction. Because at least $n/ (2R+1)$ $R$-balls are required to cover $1,\dots,n$, at least this many objects are contained in $\Theta'$.
\end{proof}

\subsection{Proof of Lemma~\ref{seqBound}} \label{seqBoundProof}
\begin{proof}
Assume $M \leq \frac{n}{3R}$. If $p_m$ denotes the probability that the $(m+1)$st object is within $R$ positions of one of the first $m$ objects, given that none of the first $m$ objects are within $R$ positions of each other, then $\frac{Rm}{n} < p_m \leq \frac{2Rm}{n-m}$ and 
\begin{align*}
P(M = m)  \geq \prod_{l=1}^{m-1} \bigg(1-\frac{2Rl}{n-l}\bigg)\frac{Rm}{n}.
\end{align*}
Taking the $\log$ we find
\begin{align*}
\log P(M = m) &\geq \log \frac{Rm}{n} +  \sum_{l=1}^{m-1} \log  \bigg(1-\frac{2Rl}{n-l}\bigg)\\
 &\geq \log \frac{Rm}{n} + (m-1) \log \bigg( \frac{1}{(m-1)}\sum_{l=1}^{m-1}  \bigg(1-\frac{2Rl}{n-l}\bigg) \bigg)\\ 
 &\geq \log \frac{Rm}{n} + (m-1) \log \bigg( 1-\frac{Rm}{n-m+1} \bigg)\\ 
 &\geq \log \frac{Rm}{n} + (m-1) \log \bigg( 1-\frac{3 Rm}{2n} \bigg) \\
  &\geq \log \frac{Rm}{n} + (m-1)  \bigg( -\frac{3 \log(2) Rm}{n} \bigg) 
\end{align*}
where the second line follows from Jensen's inequality, the fourth line follows from the fact that $m \leq \frac{n}{3R}$, and the last line follows from the fact that $(1-x) \geq \exp(-2 \log (2) x)$ for $x \leq 1/2$. We conclude that $P(M=m) \geq \frac{R}{n}m \exp\{-3 \log(2)   \frac{R}{n} m^2 \}$. Now if  $a =  \sqrt{\frac{n/R}{6 \log(2)}}$ we have 
\begin{align*}
P(M \geq a) \geq &  \sum_{m=\lceil a \rceil}^{n/(3R)-1} \frac{R}{n}m \exp\{-3 \log(2)   \frac{R}{n} m^2 \} \\
\geq & \int_{a+1}^{n/(3R)} \frac{R}{n}x \exp\{-3 \log(2)   \frac{R}{n} x^2 \} dx \\
= & \frac{1}{6 \log(2)}  \bigg( e^{-(\sqrt{6 \log(2)   R/n} +1)^2 /2}  - e^{-\log(2) n / (3R)} \bigg)
\end{align*}
where the second line follows from the fact that  $x  e^{ -\alpha x^2 /2 }$ is monotonically decreasing for $x\geq \sqrt{1/\alpha}$. Note, $P(M \geq  \sqrt{\frac{n/R}{6 \log(2)}}  )$ is greater than $\frac{1}{100}$ for $n/R \geq 7$, and $\frac{1}{10}$ for ${n/R \geq 40}$. Moreover, as $n/R \rightarrow \infty$,  ${P(M \geq \sqrt{\frac{n/R}{6 \log(2)}}) \rightarrow \frac{1}{6 \sqrt{e} \log(2)}}$.
\end{proof}

\subsection{Proof of Lemma~\ref{goodNews}} \label{goodNewsProof}
\begin{proof}
Enumerate the objects such that the first $m$ are the objects ranked amongst themselves. Let $y$ be the pairwise comparison label vector for $\sigma$ and $\hat{y}$ be the corresponding vector for $\widehat{\sigma}$. Then
\begin{align*}
\E[d_\tau(\sigma,\widehat{\sigma})] &= \sum_{k=2}^m \sum_{l=1}^{k-1} \1\{ y_{l,k} \neq \hat{y}_{l,k}\} +  \sum_{k=m+1}^n \sum_{l=1}^{k-1} \1\{ y_{l,k} \neq \hat{y}_{l,k}\}\\
&= \sum_{k=m+1}^n \sum_{l=1}^{k-1} \1\{ y_{l,k} \neq \hat{y}_{l,k}\}\\
&\leq \sum_{k=m+1}^n \sum_{l=1}^{k-1} P\{\mbox{Request } q_{l,k} | \mbox{labels to }q_{s \leq m,t \leq m} \}\\
&\leq \sum_{k=m+1}^n \sum_{l=1}^{k-1} \frac{2 a d}{m^2}\\ 
&\leq  \frac{2 a d}{m^2} \frac{(n-m)(n+m+1)}{2}\\
&\leq  a d \bigg( \frac{(n+1)^2}{m^2} - 1 \bigg).
\end{align*}
where the third line assumes that every pairwise comparison that is ambiguous (that is, cannot be imputed using the knowledge gained from the first $m$ objects) is incorrect. The fourth line follows from the application of Lemma~\ref{equiprobable} and Lemma~\ref{uniformLemma}. 
\end{proof}

\end{document}